\documentclass[pdflatex,11pt]{article}
\usepackage{amsmath,amssymb,amsfonts}
\usepackage{algorithmic}
\usepackage{graphicx}
\usepackage{textcomp}
\usepackage[table,xcdraw]{xcolor}
\usepackage[boxed,linesnumbered,ruled,lined]{algorithm2e}
\usepackage{multirow}
\usepackage{pdflscape}
\usepackage[table,xcdraw]{xcolor}
\usepackage{comment}
\usepackage{paralist}
\usepackage{natbib} 
\usepackage{amsthm}
\usepackage{authblk}

\usepackage[utf8]{inputenc}
\usepackage{algorithmic}
\usepackage{cleveref}
\usepackage{tikz}
\usetikzlibrary{positioning}
\tikzstyle{smallvertex}=[circle, draw, fill=white, inner sep=2pt]
\definecolor{iris}{rgb}{0.35, 0.31, 0.81}
\definecolor{goldenpoppy}{rgb}{0.99, 0.76, 0.0}
\definecolor{caribbeangreen}{rgb}{0.0, 0.8, 0.6}

\definecolor{h}{rgb}{0.0, 0.8, 0.6}
\definecolor{v}{rgb}{0.35, 0.31, 0.81}

\newtheorem{definition}{Definition}
\newtheorem{property}{Property}
\newtheorem{corollary}{Corollary}
\newtheorem{example}{Example} 
\newtheorem{construction}{Construction}
\newtheorem{lemma}{Lemma}
\newtheorem{theorem}{Theorem}
\Crefname{construction}{Construction}{Constructions}

\begin{document}

\title{Position Spaces and Graphs}

\author[1]{Rita-Nathalia Assaf \thanks{ritanathalia.assaf@univ-angers.fr}}
\author[1]{Tom Davot\thanks{tom.davot@univ-angers.fr}}
\author[1]{Frédéric Lardeux\thanks{frederic.lardeux@univ-angers.fr}}
\author[1]{Frédéric Saubion\thanks{frederic.saubion@univ-angers.fr}}

\affil[1]{Univ Angers, LERIA, SFR MATHSTIC, F-49000 Angers, France}

\date{}

\date{}

\maketitle

\abstract{
In this paper, we introduce position graphs, a graph-based reasoning framework based on the formalization of position spaces. This framework utilizes two strict partial orders, representing horizontal and vertical alignment and precedence, to model the relative positions of discrete tokens. Unlike general qualitative spatial calculi, position graphs are constrained by a chain condition and compatibility requirements that focus on rows and columns. We provide a comprehensive theoretical analysis of this representation, beginning with a characterization of graph consistency. Conditions to ensure the consistency of position graphs are established. Furthermore, we investigate the computational complexity of structural pattern discovery, modeled as the induced subgraph isomorphism problem. We demonstrate that this problem remains NP-complete even within the restricted class of position graphs. While initially motivated by document processing, this work focuses on the underlying mathematical properties and algebraic consistency of position-based constraints, providing a formal logical layer that is independent of specific data extraction techniques.
}

\section{Introduction}

Reasoning about spatial organization is a fundamental problem in Artificial Intelligence when one seeks to infer structures from partial, symbolic information rather than from precise metric data. Such qualitative descriptions arise in domains ranging from document layout analysis~\citep{SaoutLS24} (a motivating application for this work) to constraint-based modeling
and spatial process calculi~\citep{KnightPPV12}, and more broadly in qualitative spatial reasoning~\citep{ChenCLWOY15}. 
Unlike approaches devoted to extracting spatial relations ~\citep{Buck2022,Tian2024RelativePosition},
we assume that qualitative positional constraints are already given, and our goal is to reason about their consistency, their logical consequences, and the occurrence of given structural patterns (e.g., find table structures in documents).

Qualitative spatial reasoning provides expressive calculi for representing such knowledge. RCC8~\citep{Randell92,CohnBGG97}
is a canonical example, offering a logical language for topological relations with well-studied complexity properties~\citep{RenzN99}
and practical SAT-based reasoning techniques~\citep{GlorianLMS18}. However, many layout-like tasks rely not on region topology
but on simple alignment and precedence information, such as left-to-right or top-to-bottom relations between discrete
tokens. This motivates the introduction of a lightweight qualitative framework capturing alignment constraints with precise mathematical semantics and a direct connection to constraint satisfaction problems.

We introduce \emph{position spaces}, based on a finite set of tokens with two alignment relations, horizontal and vertical, each modeled as a strict partial order. These relations encode one-dimensional precedence without committing to numerical  coordinates. To reflect the intended semantics of rows and columns, we impose a local no-branching restriction: the successors (resp. predecessors) of a token along an alignment must form a chain. Furthermore, the two alignments must be compatible, in the sense that they admit a coherent two-dimensional interpretation. A position space satisfying these  conditions is said to be consistent. Unlike dimension‑2 posets or products of chains, our notion of a position space introduces two independent strict partial orders whose interactions are governed by new exclusivity and non‑overlapping conditions, yielding a qualitatively new formal object. Moreover, we define a minimal set of algebraic constraints ensuring coherent two‑dimensional alignment, with regard to a canonical interpretation. 

Position spaces admit a natural representation as labeled directed graphs, called \emph{position graphs}, whose arcs encode horizontal or vertical precedence in the spirit of classical order graph-based representations. This enables two central 
reasoning tasks. First, we characterize consistency through forbidden patterns in mixed-label paths and provide a linear-time  
algorithm that constructs a valid row/column assignment whenever the space is consistent.  

Second, we study the discovery of patterns: given two position graphs, we ask whether the smaller appears as an induced, label-preserving subgraph of the larger. We prove that induced subgraph isomorphism remains NP-complete even within the restricted class of position graphs, establishing a hardness boundary that complements the tractability of consistency checking.

\paragraph{Contributions.}
The paper proposes the following contributions :
\begin{itemize}
\item 
 the formal definition of position spaces with no-branching and compatibility conditions capturing coherent row/column semantics,
\item  a labeled graph representation and a characterization of consistency via forbidden mixed-label paths,
\item 
a linear-time row/column assignment algorithm that allows us to check the consistency of a position space,
\item  an NP-completeness result for induced subgraph isomorphism on position graphs.
\end{itemize}

Note that we do not aim to present a new layout engine/solver for practical cases, as in \citep{SaoutLS23}. In contrast, the present work focuses on the essential theoretical foundations—consistency proofs and algorithmic complexity—for such an application. In particular, the identification of NP-completeness for pattern matching in position graphs provides a necessary theoretical constraint for practitioners, justifying the move toward the heuristic-based approaches. 

\paragraph{Organization.}
Section~\ref{sec:position_space} introduces position spaces and position graphs. In Section \ref{sec:csp_pos}, we study consistency from an algorithmic point of view. Section~\ref{sec:complexity} formalizes pattern discovery and proves NP-completeness.
Section~\ref{sec:related_works} discusses related work in ordered structures, graph theory, qualitative reasoning and constraint-based modeling.

\section{Position Space}
\label{sec:position_space}
We actually adopt a horizontal/vertical positioning point of view and hence, we rely on partial order relations to define the relative positioning of objects/items/tokens. An alignment relation defines how elements are aligned and in what order they are aligned to capture usual notions of row/column. 

\subsection{Position Space} 

\begin{definition}[Alignment Relation, see \Cref{fig:def_align}]
  Given a set of tokens ${\cal T}$, an alignment relation~$<$ defined on ${\cal T} \times {\cal T}$ is a strict partial order such that $\forall t,t',t'' \in \mathcal{T}$%\FL{tous différents ?}
  , we have 
  \begin{itemize}
    \item $(t< t' \wedge t< t'') \Rightarrow (t' < t'' \vee t'' < t')$, and
    \item $(t'< t \wedge t''< t) \Rightarrow (t' < t'' \vee t'' < t')$.
  \end{itemize}
  \label{def:al_rel}
\end{definition}

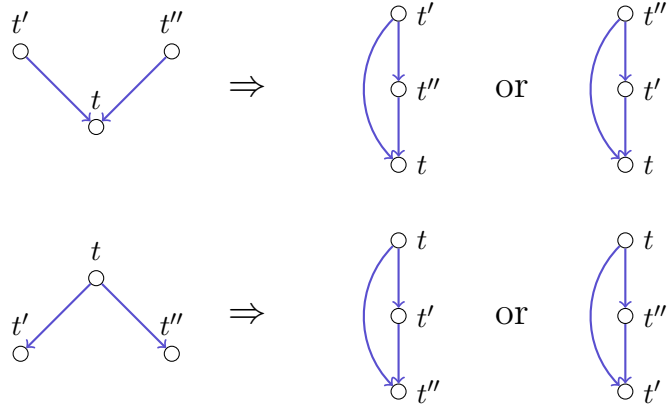
\begin{figure}[h]
\centering 
\begin{tikzpicture}
  \foreach[count=\i from 1] \x/\y/\l in {0/0/$t'$,1/-1/$t$,2/0/$t''$,0/-4/$t'$,1/-3/$t$,2/-4/$t''$}
  \node[smallvertex,label=above:{\l}] (\i) at (\x,\y) {};

  \foreach \a/\b in {1/2,3/2,5/4,5/6}
  \draw[thick,->,draw=v] (\a)--(\b);

  \foreach \I/\Y/\X in {{$t'$,$t''$,$t$}/0/1,{$t''$,$t'$,$t$}/0/2,{$t$,$t'$,$t''$}/-3/1,{$t$,$t''$,$t'$}/-3/2}{
    \foreach[count=\y from 1] \J in \I {
      \node[smallvertex,label=right:{\J}] (B\y) at (2+\X*3,1.5-\y+\Y) {};
    }
    \draw[thick,->,draw=v] (B1)--(B2);
    \draw[thick,->,draw=v] (B2)--(B3);
    \draw[thick,->,draw=v] (B1) to[out=225,in=135] (B3);
  }
  \foreach \x/\y/\l in {3/-0.5/$\Rightarrow$,3/-3.5/$\Rightarrow$,6.5/-0.5/or,6.5/-3.5/or}
  \node at (\x,\y) {\Large \l};

\end{tikzpicture}

\caption{\label{fig:def_align}Illustration of conditions satisfied by an alignment relation.}
\end{figure}

A \emph{chain} of an alignment relation $<$ on a set of tokens $\mathcal{T}$ is a totally ordered subset of $\mathcal{T}$, \textit{i.e.} a sequence of tokens $(t_1,\dots,t_k)$ such that for all $i\in [k-1]$, $t_i<t_{i+1}$, where $[k-1]$ is the integer interval $[1,\cdots,k-1]$. The conditions introduced in \Cref{def:al_rel} prevent the presence of inconsistent alignments from our point of view, meaning alignments that correspond to an incongruent division of the same row/column. Note that we consider a strict partial order relation. Although $<$ is formally a strict partial order, our intended semantics for an alignment is one-dimensional: horizontal (resp. vertical) alignment represents left-to-right (resp. top-to-bottom) ordering within a single row (resp. column). We get the following property immediately.

\begin{property}[Chain Condition]
For every token $t\in {\cal T}$ and an alignment relation~$<$ on $\mathcal{T}$, both the set of successors $\{u \in {\cal T}\mid t<u\}$ and the set of predecessors $\{u \in {\cal T}\mid u<t\}$ are chains.  
\end{property}

\begin{definition}[Position Space]
A position space is a triplet $({\cal T},<_h,<_v)$ such that $<_h$ (horizontal alignment) and $<_v$ (vertical alignment) are two alignment relations defined on ${\cal T}$.
\end{definition}

A  position space defines the organization of tokens according to their vertical and horizontal alignments. We assume that every element of ${\cal T}$ must have a relative positioning with respect to at least one other element (either vertical or horizontal). Note that this latter assumption is not mandatory in our framework. If one wishes to relax this constraint, it suffices to consider reflexive relations, which does not fundamentally change the rest of the discussion.

\subsection{Consistency of a Position Space}

As mentioned in the introduction, we aim to ensure that our alignment relations represent horizontal and vertical positions that are coherent with their canonical interpretation. Hence, in this section, we precisely define a compatibility property to ensure the consistency of a position space.  

\begin{definition}[$v$-chain]
\label{def:vchain}
Let $\Pi=(T,<_{h},<_{v})$ be a position space. A sequence of tokens $(t_1,\dots,t_k)$
is called a \emph{$v$-chain} if the following two conditions hold:
\begin{enumerate}
    \item for every $i \in \{1,\dots,k-1\}$, at least one of the following holds:
    \[
        t_i <_{h} t_{i+1}, \qquad
        t_{i+1} <_{h} t_i, \qquad
        t_i <_{v} t_{i+1};
    \]
    \item there exists at least one index $i \in \{0,\dots,k-1\}$ such that
    \[
        t_i <_{v} t_{i+1}.
    \]
\end{enumerate}
An \emph{$h$-chain} is defined symmetrically.
\end{definition}

Intuitively, in a v-chain, consecutive tokens are either related by $<_h$, regardless of the direction, or ordered by $<_v$, with at least one pair ordered by $<_v$. An \emph{h-chain} is defined symmetrically.

\begin{definition}[Compatibility of alignment relations]
\label{def:compatibility}
Given a position space $\mathcal{P}=({\cal T},<_h,<_v)$, $<_h$ and $<_v$ are said to be compatible if and only if for any
pair of tokens $t,t' \in \mathcal{T}$ such that $t <_h t'$ (resp. $t<_v t'$), there is no $v$-chain (resp. $h$-chain) between $t$ and $t'$.
\end{definition}

Observe that this condition ensures that two tokens can be ordered by at most one alignment relation.

\begin{definition}
\label{def:consist_space}
A position space $({\cal T},<_h,<_v)$ whose relations are compatible is said to be consistent.
\end{definition}

\begin{figure}[htbt!]
  \centering 
  \begin{tikzpicture}
    \foreach[count=\i from 4] \x/\y/\l/\p in {4/0/$t_1$/above,6/0/$t_2$/above,4/-1/$t_3$/below,6/-1/$t_4$/below,
    8/0/$t_1$/above,10/0/$t_2$/right,8/-1/$t_3$/below,10/1/$t_4$/above,12/1/$t_5$/above,12/-1/$t_6$/below}
    \node[smallvertex,label=\p:{\l}] (\i) at (\x,\y) {};

    \foreach \a/\b/\c in {4/5/h,4/6/v,5/7/v,7/6/h,8/9/h,11/9/v,11/12/h,12/13/v,13/10/h,8/10/v}
    \draw[thick,->,draw=\c] (\a)--(\b);

  \end{tikzpicture}

\caption{Illustration of inconsistent position spaces. Vertical alignments are depicted with a blue arc while horizontal alignments are depicted with a green arc. In all cases, consistency is violated because $t_1 <_v t_3$ and there is a h-chain between $t_1$ and $t_3$ containing a horizontal alignment relation. \label{fig:Overlap}}
\end{figure}
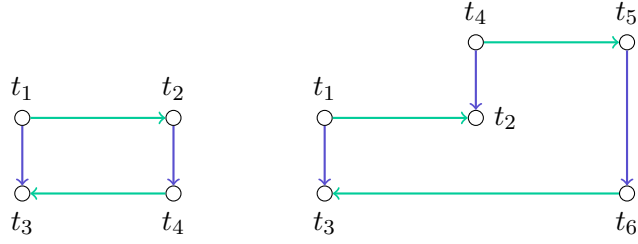

Figure~\ref{fig:Overlap} illustrates two basic cases of inconsistent position spaces. 
Note that grid or lattice graphs may share some similarities with our approach; however, such graphs are generally regular and consider integer-coordinate points as vertices, with edges representing pairs of vertices that share one coordinate. Here, we aim for a more specific and declarative representation based on alignments.

\subsection{Position Graph}
Graphs constitute a natural representation of partial ordering relations (e.g., Hasse diagrams \citep{Hasse1952} that also avoid transitivity direct  representation). Moreover, as pointed out in the introduction, one of our aims is to identify particular patterns within position spaces, and graphs are a natural framework for modeling this problem. 

\begin{definition}[Position Graph]
Given a position space $\Pi=({\cal T},<_h,<_v)$, the induced position graph $GP_\Pi$ is defined as a labeled directed graph $({\cal T},{\cal A},{\cal E})$ where ${\cal A} \subseteq {\cal T} \times {\cal T}$ and ${\cal E} : \mathcal{T} \times \mathcal{T} \to \{h,v\}$ such that:
\begin{itemize}
\item $\forall t,t' \in {\cal T}, t <_h t' \Leftrightarrow (t,t') \in {\cal A} \wedge {\cal E}(t,t') = h$, and
\item $\forall t,t' \in {\cal T}, t <_v t' \Leftrightarrow (t,t') \in {\cal A} \wedge {\cal E}(t,t') = v$.
\end{itemize}
\label{def:GP}
\end{definition}

\begin{example}
\label{exemple}
Let ${\cal T}=\{1,2,3,4,5,6,7,8,9,10,11\}$. We define our position space by setting:
$1 <_h 2, 2 <_h 3, 3 <_h 4, 5 <_h 6, 6 <_h 7, 8 <_h 9, 9 <_h 10$,
$1 <_v 5, 5 <_v 8, 2 <_v 6, 6 <_v 9, 9 <_v 11, 3 <_v 7, 7 <_v 10$.
\end{example}

For \Cref{exemple}, we obtain the graph shown in Figure \ref{fig:GP}. In Figure \ref{fig:GP}, we omit the labels of the arcs because those labeled 'h' will be drawn in green and horizontally and those labeled 'v' in blue and vertically.

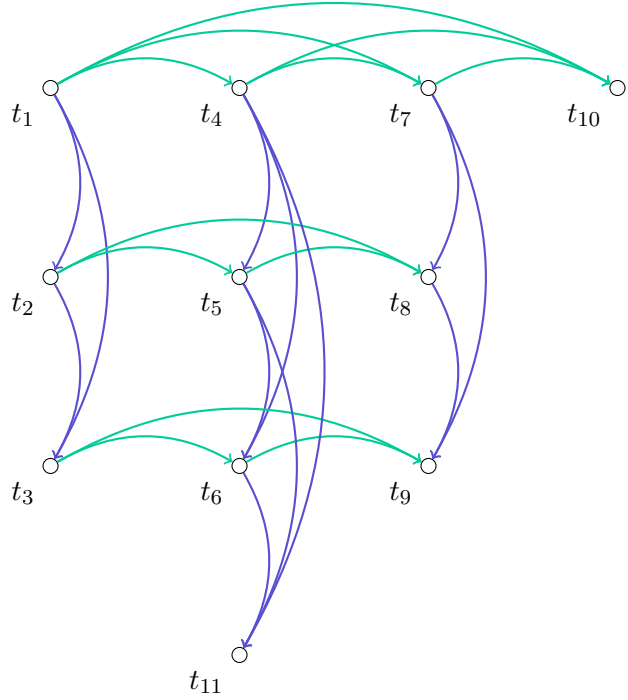
\begin{figure}[htbt]
\centering
\begin{tikzpicture}
  \newcounter{clab}
  \setcounter{clab}{0}
%    You can now write \verb|\themycounter| to obtain \themycounter.

  \foreach \x in {1,...,3}
  \foreach \y in {1,...,3}{
    \stepcounter{clab}
    \node[smallvertex,label=-135:{$t_{\theclab}$}] (\x\y) at (2.5*\x,-2.5*\y) {};
    \foreach \i in {1,...,\x}{
      \ifthenelse{\equal{\i}{\x}}{}{
        \draw[thick,->,draw=h] (\i\y) to[bend left] (\x\y);
      }
    }
    \foreach \i in {1,...,\y}{
      \ifthenelse{\equal{\i}{\y}}{}{
        \draw[thick,->,draw=v] (\x\i) to[bend left] (\x\y);
      }
    }
  }
  \node[smallvertex,label=-135:{$t_{10}$}] (10) at (2.5*4,-2.5) {};
  \foreach \i in {1,2,3} \draw[thick,->,draw=h] (\i1) to[bend left] (10);
  \node[smallvertex,label=-135:{$t_{11}$}] (11) at (2.5*2,-2.5*4) {};
  \foreach \i in {1,2,3} \draw[thick,->,draw=v] (2\i) to[bend left] (11);
\end{tikzpicture}

\caption{Position graph $GP$.\label{fig:GP}}
\end{figure}

\begin{definition}[Consistent Position Graph]
A position graph is said to be consistent if and only if its corresponding position space is consistent. 
\label{def:consistency}
\end{definition}

A $v$-\emph{cycle} is a sequence of tokens $(t_1,\dots,t_k=t_1)$ such that there is an arc with label $v$ between $t_i$ and $t_{i+1}$ for some $i\in[k-1]$ and for all $i \in [k-1]$, either $(t_i,t_{i+1}) \in \mathcal{A}$ or $(t_{i+1},t_i)\in \mathcal{A}$ with $\mathcal{E}(t_{i+1},t_i)=h$. A $h$-\emph{cycle} is defined symmetrically. 
\Cref{def:compatibility} yields the following characterization.

\begin{property}[Characterization of a Consistent Position Graph]
\label{prop:consitent_graph}
A position graph $GP_\Pi= (\mathcal{T},\mathcal{A},\mathcal{E})$ is consistent if and only if it does not contain a $v$-cycle or an $h$-cycle.
\end{property}

\begin{proof}
  Let $(\mathcal{T},<_h,<_v)$ be a consistent position space and let $GP_\Pi$ be its corresponding position graph. Let $t$ and $t'$ be two tokens such that $t <_h t'$ (resp. $t<_v t'$). If there is a $v$-chain (resp. h-chain) $(t_1=t,\dots,t_k=t')$ between $t$ and $t'$, then there is a $v$-cycle (resp. $h$-cycle) $(t_1=t,\dots,t_k=t',t)$ in $GP_\Pi$ since there is an arc labelled $h$ (resp. $v$) between $t$ and $t'$. Moreover, if there is a $v$-cycle (resp. $h$-cycle) $(t_1=t,\dots,t_{k-1}=t',t_k=t)$ in $GP_\Pi$, then there is a $v$-chain (resp. $h$-chain) $(t_1=t,\dots,t_{k-1}=t')$ in $\Pi$. Hence, $(\mathcal{T},<_h,<_v)$ does not contain a $v$-chain (resp. $h$-chain) if and only if $GP_\Pi$ does not contain a $v$-cycle (resp. $h$-cycle).  
  
\end{proof}

\begin{comment}
   
\begin{corollary}
\label{corr:incons}
In an inconsistent position graph $GP_\Pi=({\cal T},{\cal A},{\cal E})$, 
 $\forall (t,t') \in {\cal A}$, if there exists another directed path from $t$ to $t'$ in  $GP_\Pi$ then this path contains at least two arcs $(u,u'),(v,v') \in {\cal A}$ such that ${\cal E}(t,t') \neq {\cal E}(u,u')$ and ${\cal E}(t,t') \neq {\cal E}(v,v')$. 
\end{corollary}

\begin{proof}
We proceed by contradiction. Let us consider $(t,t') \in {\cal A}$, such that there exists an alternative path another directed path from $t$ to $t'$ in  $GP_\Pi$ that contains a unique arc $a=(t_{a_{in}},t_{a_{out}}) \in {\cal A}$ such that ${\cal E}(t,t') \neq {\cal E}(a)$. This path can be written as a sequence of directed arcs. 

$$(t,t_1),(t_1,t_2),\cdots, (t_i,t_{a_{in}}),(t_{a_{in}},t_{a_{out}}),(t_{a_{out}},t_{i+1}),\cdots,(t_n,t').$$

Due to the transitivity of the alignment relations, since all arcs, except $a$, have the same label, there exists necessarily a shortest path 

$$(t,t_{a_{in}}),(t_{a_{in}},t_{a_{out}}),(t_{a_{out}},t').$$ 

Due to Definition \ref{def:al_rel} and the fact that $(t,t') \in {\cal A}$, we get, by generalization to $4$ tokens, that $t,t',t_{a_{in}},t_{a_{out}}$ must be ordered. Hence, there exists either an arc $(t_{a_{in}},t_{a_{out}})$ or $(t_{a_{out}},t_{a_{in}})$ whose label is identical to the label of $(t,t')$, which is contradictory with the assumption. Since the alternate path cannot contain a single arc with a different label, it necessarily contains at least two such arcs. 
\end{proof}
\end{comment}

In the next section, we study the consistency of the position space from an algorithmic perspective. We propose an algorithm that checks if a position space is consistent by trying to assign row and column indices to each element of the position space.

%Remind that our main concern is to identify patterns (i.e., specific organization of some tokens) within the position space, and thereby in the position graph, turning pattern search into a subgraph isomorphism problem. Hence, the next section is devoted to the computational study of this problem in our specific context. 

\section{Rows/Columns Embedding of Position Spaces}
\label{sec:csp_pos}

We now address the Rows/Columns (R/C) embedding problem, which consists of assigning concrete positions to tokens. Intuitively, given a consistent position space  $\Pi=(\mathcal{T},<_h,<_v)$, we want to introduce notions of rows and columns such that the alignment relations $<_h$ and $<_v$ are respected: if $t <_h t'$, then $t$ and $t'$ should be on the same row with $t'$ to the right of $t$ and if $t <_v t'$, then $t$ and $t'$ should be on the same column with $t'$ below $t$. We first define rows and columns with regard to the notion of chains previously defined.

\begin{definition}[Rows and Columns]
Given a position space $\Pi=({\cal T},<_h,<_v)$, a \emph{row}  (resp. a \emph{column}) is a maximal chain contained in $<_h$ (resp. $<_v$). $\mathcal{R}_\Pi$ and $\mathcal{C}_\Pi$ are the sets of rows and columns of $\Pi$, respectively.
\label{def:LC}
\end{definition}

Observe that a token necessarily belongs to a column and a row: a token not comparable in $<_h$ (resp. $<_v$) forms a singleton row (resp. column). Moreover, a token cannot belong to two distinct rows (resp. columns), as this would either contradict the transitivity of $<_h$ (resp. $<_v$) or the maximality of the chains. Hence, the following property follows. 

\begin{property}
  \label{prop:LC partition}
Given a consistent position space $\Pi=({\cal T},<_h,<_v)$, the set of rows (resp. columns) ${\cal R}_\Pi$ (resp. ${\cal C}_\Pi$) forms a partition of ${\cal T}$. 
\end{property}

Given a consistent position graph $GP_\Pi = (\mathcal{T},\mathcal{A},\mathcal{E})$, we define the graph of columns $G^\Pi_C$ of $GP_\Pi$ obtained by contracting each column into a single vertex. Formally,

\[G^\Pi_C = (\mathcal{C}_\Pi,\{(C,C') \mid \exists t \in C, t' \in C', (t,t') \in \mathcal{A}\}). \]
The graph of rows of $G^\Pi_R$ of $GP_\Pi$ is defined symmetrically.
Note that the rows % (maximal chains in \(\prec_h\))
and columns % (maximal chains in \(\prec_v\))
as in Definition \ref{def:LC} can be computed in linear time in the input size, i.e., $O(|{\cal T}|+|{\cal A}|)$ by computing connected components in the subgraphs induced by h-arcs and v-arcs, respectively. % (when considered as undirected graphs)

% $\mathcal{R}_\Pi$ and $\mathcal{C}_\Pi$ are assumed to be totally ordered sets\footnote{Without loss of generality $\mathcal{R}_\Pi$ and $\mathcal{C}_\Pi$ can be assumed to be isomorphic to finite subsets of $\mathbb{N}$.} using an ordering relation  $<_{\mathcal{R}_\Pi}$ and $<_{\mathcal{C}_\Pi}$ (Note that subscript will be omitted in the remainder of the paper, when clear from context). 

\begin{definition}[R/C embedding]
\label{def:RCEmb}
Given a position space $\Pi=(\mathcal{T},<_h,<_v)$, a \emph{R/C embedding} of $\Pi$ is a mapping $\mathcal{T} \to |{\mathcal{R}_\Pi}| \times |\mathcal{C}_\Pi| $ that assigns to each token $t\in \mathcal{T}$ a row number $x_t$ and a column number $y_t$ such that:
\begin{itemize}
  \item $x_t = x_{t'}$ if and only if $t <_h t'$ or $t' <_h t$,
  \item if $t <_h t'$, then $y_t < y_{t'}$, 
  \item $y_t = y_{t'}$ if and only if $t <_v t'$ or $t' <_v t$, and
  \item if $t <_v t'$, then $x_t < x_{t'}$.
\end{itemize}
\end{definition}

\begin{property}
If a position space $\Pi$ is inconsistent, then it does not admit an R/C embedding.
\end{property}

%\begin{proof}
 % Let $\Pi$ be an inconsistent position space and suppose, toward a contradiction, that $\Pi$ admits an R/C embedding. Without loss of generality, suppose $\Pi$ contains two tokens $t$ and $t'$ such that $t <_h t'$ and there exists a $v$-chain $(t_0=t,\dots,t_k=t')$ such that $t_i <_h t_{i+1}$ for some $i \in [k-1]$ (see Definition \ref{def:consist_space} and Definition \ref{def:compatibility}). Since $t<_h t'$, $t'$ belongs to the same row, that is, $x_t=x_{t'}$. Observe that for any $j \in [k-1]$, we have $x_{t_i}=x_{t_{i+1}}$ if $t_i<_h t_{i+1}$ or $x_{t_i}< x_{t_{i+1}}$ if $t_i<_v t_{i+1}$. Hence, by induction we have that $x_t<x_{t_{i+1}}\leq x_{t'}$ which is a contradiction.
%\end{proof}

\begin{proof}
Assume, toward a contradiction, that $\Pi=(T,<_{h},<_{v})$ is inconsistent and admits an R/C embedding. Since $\Pi$ is inconsistent, by Definition \ref{def:compatibility} there are two possible cases. We treat one of them, the other being symmetric. Suppose there exist two tokens $t,t' \in T$ such that $t <_{h} t'$ and there exists a $v$-chain $(t_0,\dots,t_k)$ with $t_0=t$ and $t_k=t'$.

Since $t <_{h} t'$, we must have $x_t = x_{t'}$. Now consider any consecutive pair $(t_i,t_{i+1})$ in the $v$-chain. By definition of a $v$-chain, one of the following holds:
\begin{itemize}
    \item $t_i <_{h} t_{i+1}$, in which case $x_{t_i}=x_{t_{i+1}}$;
    \item $t_{i+1} <_{h} t_i$, in which case again $x_{t_i}=x_{t_{i+1}}$;
    \item $t_i <_{v} t_{i+1}$, in which case $x_{t_i}<x_{t_{i+1}}$.
\end{itemize}

Therefore, for every $i \in \{0,\dots,k-1\}$, we have $x_{t_i} \le x_{t_{i+1}}$.
Moreover, since $(t_0,\dots,t_k)$ is a $v$-chain, there exists at least one index $i$ such that
$t_i <_{v} t_{i+1}$, and hence $x_{t_i} < x_{t_{i+1}}$. We get $x_t = x_{t_0} < x_{t_k} = x_{t'},
$ which contradicts $x_t=x_{t'}$. The case where $t <_{v} t'$ and there exists an $h$-chain from $t$ to $t'$ is symmetric, using the column numbers $y_t$ instead of the row numbers $x_t$. Hence, $\Pi$ does not admit an R/C embedding.
\end{proof}

\begin{algorithm}
\caption{R/C embedding of a position space}
\label{alg:linecol}
\DontPrintSemicolon

\KwIn{Position Graph $GP_\Pi=({\cal T},{\cal A},{\cal E})$}

\KwOut{
  A R/C embedding for $\Pi$ or \textsc{Inconsistent}
}
\If{$G^\Pi_C$ or $G^\Pi_R$ contain a directed cycle}{
  \Return{ \textsc{Inconsistent}}
}

 $T_R = \{R_1,\dots,R_{|\mathcal{R}_\Pi|}\} \gets$ topological ordering of $R_\Pi$\;
\ForEach{$R_i \in T_R$}
{ 
  \ForEach{$t \in R_i$}{
    $x_t \gets i$ 
  }
}
 $T_C = \{C_1,\dots,C_{|\mathcal{C}_\Pi|}\} \gets$ topological ordering of $C_\Pi$\;
\ForEach{$C_i \in T_C$}  
{ 
  \ForEach{$t \in C_i$}{
    $y_t \gets i$ \;
  }
}

\end{algorithm}

\begin{lemma}
  \label{lemma:alg inconsistent}
  A position graph $GP_\Pi = (\mathcal{T},\mathcal{A},\mathcal{E})$ is inconsistent if and only if either its graph of columns or its graph of rows contains a directed cycle.
\end{lemma}

\begin{proof}
  Suppose without loss of generality that there is a cycle $(C_1,\dots,C_k=C_1)$ in $G^\Pi_C$. Then, by construction of $G^\Pi_C$, for all $i \in [k-1]$, there is a token $t^{out}_i\in C_i$ and a token $t^{in}_{i+1} \in C_{i+1}$ such that $\mathcal{E}(t^{out}_i,t^{in}_{i+1})=h$. Moreover, for each $\in [k]$, we have that $\mathcal{E}(t^{in}_i,t^{out}_i)=v$ or $\mathcal{E}(t^{out}_i,t^{in}_i)=v$. Thus, the sequence $(t^{out}_1,t^{in}_2,t^{out}_2,t^{in}_3,\dots,t^{in}_1,t^{out}_1)$ is an $h$-cycle and so, $GP_\Pi$ is inconsistent.

 Now suppose that $GP_\Pi$ is inconsistent. And suppose without loss of generality that $GP_\Pi$ contains an $h$-cycle $(t_1,\dots,t_k=t_1)$. For $w=(C_1,\dots,C_k)$ be the sequence of columns such that $t_i \in C_i$ for all $i \in [k]$. Clearly, if $C_i \neq C_{i+1}$, then $\mathcal{E}(t_i,t_{i+1})=v$ by construction of $G^\Pi_C$, and so $(C_i,C_{i+1})$ is an arc in $G^\Pi_C$. Hence, $w$ is a closed walk in $G^\Pi_C$ and thus, $G^\Pi_C$ contains a cycle. 
 
\end{proof}

It follows that if $GP_\Pi$ is consistent, then we can order the vertices $G^\Pi_C$ and $G^\Pi_R$ and use that order to produce an R/C embedding for $\Pi$.

\begin{theorem}
  \label{theorem:grid embedding} If $GP_\Pi$ is consistent, then \Cref{alg:linecol} produces an R/C embedding for $\Pi$ in $\mathcal{O}(|\mathcal{T}|+|\mathcal{A}|)$ time.
\end{theorem}

\begin{proof}
  First, by \Cref{prop:LC partition}, each token belongs to a column and a row and thus, each token $t$ receives values $x_t$ and $y_t$ on lines 7 and 13.
  Further, we need to show that the mapping is an R/C embedding.
  Let $t$ and $t'$ be any pair of tokens such that $t <_h t'$. We want to show that $y_t<y_{t'}$ and $x_t=x_{t'}$. By definition, $t$ and $t'$ belong to the same row and since the algorithm gives the same value $x_t$ to all tokens of a row (line 7), we have $x_t=x_{t'}$. Suppose for contradiction that there are two tokens $t$ and $t'$ such that $t <_h t'$ and $y_{t'}\leq y_t$. Let $C$ and $C'$ be the columns containing $t$ and $t'$, respectively. Since $t<_h t'$, there is an arc between $C$ and $C'$. If $C =C'$, then there is a loop in $G^\Pi_C$, which is a contradiction. So, $C\neq C'$ and we necessarily have that $y_t\neq y_{t'}$ since the algorithm selects one column at a time. Since $y_{t'} < y_t$, the column $C'$ has been selected before $C$ by the algorithm. But then, since there is an arc between $C$ and $C'$, it contradicts $G^\Pi_C$ being acyclic. Hence, we cannot have $y_{t'}\leq y_t$. The case where $t <_v t'$ is symmetric.   Finally, since the algorithm assigns a different value $y_t$ (resp. $x_t$) to two tokens if they do not belong to the same column (resp. row), it follows that the produced mapping is an R/C embedding.

  Regarding the complexity, the construction of $G^\Pi_C$ and $G^\Pi_R$ can be done in $\mathcal{O}(|\mathcal{T}|+|\mathcal{A}|)$. Detecting a cycle in $G^\Pi_C$ or $G^\Pi_R$ can be done in $\mathcal{O}(|\mathcal{T}|+|\mathcal{A}|)$ using a Breadth-First Search. Computing a topological ordering can be done in $\mathcal{O}(|\mathcal{T}|+|\mathcal{A}|)$ using Kahn's algorithm. We obtain an overall complexity of $\mathcal{O}(|\mathcal{T}|+|\mathcal{A}|)$.
\end{proof}

We may now fully characterize consistent position spaces with regard to the induced notions of rows and columns. 

\begin{corollary}
There is a R/C embedding for $\Pi$ if and only if $\Pi$ is consistent.
\end{corollary}

At this step, the notion of position space has been precisely defined from its initial description by means of ordering relations to its graph representation. We are now able to use it for a specific task, already mentioned: the discovery of a particular token pattern.

\section{Subgraph Isomorphism}
\label{sec:complexity}

\newcommand{\GP}{\ensuremath{GP_\Pi=({\cal T},{\cal A},{\cal E})}\xspace}
\newcommand{\pattern}{\ensuremath{G_\Pi\xspace}}
\newcommand{\target}{\ensuremath{H_\Pi\xspace}}
\newcommand{\TSAT}{\textsc{3-SAT}\xspace}

\newcommand{\probdef}[3]{%
  \begin{center}
  \fbox{%
    \begin{minipage}{0.93\linewidth}
      \textbf{\sc #1} \\[0.5em]
      \textbf{Input :} #2 \\[0.3em]
      \textbf{Question :} #3
    \end{minipage}
  }
  \end{center}
}

Position graphs provide a natural representation of alignment relations and, as such, can be used to detect structural configurations of interest. As discussed earlier, our primary motivation is the identification of specific arrangements of tokens that may correspond to meaningful document structures, such as tables in document analysis \citep{SaoutLS23}. Despite the additional structural constraints imposed by position graphs, we show in the next section that the problem of searching for a pattern in such graphs remains NP-complete.

To formalize this task, let $G$ and $H$ be two position graphs. We are interested in determining whether the pattern graph $H$ occurs in the target graph $G$. The graph $H$ is itself a position graph, typically smaller than $G$, whose vertices should be understood as placeholders to be matched with actual tokens of $G$.

For the complexity proofs that follow, it is convenient to use standard graph-theoretic notation, since vertices and arcs will encode elements related to Boolean formulas rather than document tokens directly. Accordingly, for a graph $X$, we denote by $V(X)$ its vertex set and by $A(X)$ its arc set. Two directed graphs $G$ and $H$ are said to be isomorphic if there exists a bijection $g: V(H) \to V(G)$ such that, for all $u,v \in V(H)$, $(u,v) \in A(H) \iff (g(u),g(v)) \in A(G)$. In the specific case of position graphs, the isomorphism must additionally preserve arc labels. More precisely, if $G$ and $H$ are position graphs with labeling functions $E_G : A(G) \to \{h,v\}$ and $E_H : A(H) \to \{h,v\}$, then an isomorphism $g$ must satisfy $E_H(u,v) = E_G(g(u),g(v))$ for every arc $(u,v) \in A(H)$. In other words, both adjacency and the type of alignment relation (horizontal or vertical) must be preserved. In summary, such a bijection $g$ allows us to identify a subgraph in $G$ that is isomorphic to $H$.  

\medskip 
Finding a specific pattern (the graph $H$) in a given position space (the target graph $G$), corresponds then to the Induced Subgraph Isomorphism problem formulated as :

 \probdef{Induced Subgraph Isomorphism}{two graphs $G$ and $H$.}{does $G$ contain an induced subgraph isomorphic to $H$}

Remember that a general subgraph of $G$ may omit some edges between chosen vertices. In contrast, an induced subgraph cannot omit them (it preserves both adjacency and non-adjacency within $A(G)$). We model pattern discovery as induced subgraph isomorphism. In position graphs, arcs represent precedence constraints (horizontal/vertical alignments), and the absence of an arc is itself informative, expressing non-comparability/no direct precedence within the extracted constraint network. Consequently, a correct occurrence of a structural template must preserve not only the required relations but also the non-relations among the matched tokens. This is exactly captured by induced subgraph isomorphism, which preserves adjacency and non-adjacency, unlike general subgraph isomorphism that may ignore extra host edges among selected vertices.

\medskip
In the following, when constructing an isomorphism $g$, we say that we identify a path $(v_1,\dots,v_i)$ of $H$ to a path $(u_1,\dots,u_i)$ of $G$  by setting $g(v_j)=u_j$ for all $j\leq i$. A \emph{transitive vertical path} (resp. \emph{transitive horizontal path}) 
$(v_1,\dots,v_i)$ is a transitively closed directed path in which every arc is labeled $v$ (resp. $h$).

\medskip

In the remainder of this section, we show that {\sc Induced Subgraph Isomorphism} is NP-complete when $G$ and $H$ are consistent position graphs. Let us recall that \TSAT is the decision problem of determining whether a CNF Boolean formula, with each clause containing exactly three literals (variables or negations), is satisfiable. \TSAT is known to be NP-complete~\citep{Cook1971}.

 % Let $P^j_i$ denote the caterpillar constituted by a path $P_n = (v_1,\dots,v_n)$ and an
\begin{construction}[Target Graph, see \Cref{fig:target-construction}]
  \label{const:target}
   Let $\varphi$ be a \TSAT 
formula  with Boolean variables set $\{x_1,\dots,x_n\}$ and clauses set $\{Q_1,\dots,Q_m\}$.
We construct a target position graph $G$ with vertex set
$$V(G) = \{o,q_i,q^*_i,q^p_i,u_j,f^i_j,t^i_j,f^*_j,t^*_j,\ell^i_j,\bar{\ell}^i_j \mid p\in [3],  i \in [m] ,j \in [n]\}.$$

\begin{itemize}
  \item Introduce a transitive horizontal path $A=(o,q_1,\dots,q_m)$ and a transitive vertical path $B=(o,u_1,\dots,u_n)$.
  \item For each clause $Q_i$, 
  \begin{itemize}
    \item introduce the transitive horizontal path $(q^*_i,q^1_i,q^2_i,q^3_i)$, 
    \item introduce the vertical arc $(q_i,q^*_i)$, and
    \item for each $k \leq 3$, let $x_j$ be the $k^{th}$ variable occurring in $Q_i$, introduce the vertical arc $(q^k_i,\ell^i_j)$ if $x_j$ appears positively in $Q_i$ and the vertical arc $(q^k_i,\bar{\ell}^i_j)$ otherwise.
  \end{itemize}
  \item For each variable $x_j$, 
  \begin{itemize}
    \item introduce a transitive horizontal path $(u_j,f^*_j,t^*_j)$, and 
    \item introduce two transitive vertical paths $F_j=(f^*_j,f^1_j,\dots,f^m_j)$ and $T_j=(t^*_j,t^1_j,\dots,t^m_j)$, 
  \end{itemize}
  \item For each clause $Q_i$ and each variable $x_j$, introduce the two horizontal arcs $(t^i_j,\bar\ell^i_j)$ and $(f^i_j,\ell^i_j)$.
\end{itemize}
 \end{construction}

\begin{figure}
 \scalebox{0.85}{
  \begin{tikzpicture}
   
   \node[smallvertex,label=above:{$o$}] (z) at (0,1) {};
   \foreach \x in {1,2,3}{
     \node[smallvertex,label=90:{$q_\x$}] (qx\x) at (-2+4*\x,1) {};
     \node[smallvertex,label=left:{$q^*_\x$}] (q\x0) at (-2+4*\x,0) {};
     \draw[thick,->,draw=v] (qx\x) -- (q\x0) ;
     \foreach[count=\j from 0] \i in {1,2,3}{
       \node[smallvertex,label=90:{$q^\i_\x$}] (q\x\i) at (-2+4*\x+\i,0) {};
       \draw[thick,->,draw=h] (q\x\j) -- (q\x\i) ;
     }
   }
   \def\s{0.7}
   \def\Y{3.5}
   \foreach \y in {1,2,3}{
     \node[smallvertex,label=left:{$u_\y$}] (u\y) at (0,\Y-\y*7*\s) {};
     \node[smallvertex,label=above:{$f^*_\y$}] (f\y0) at (1,\Y-\y*7*\s) {};
     \node[smallvertex,label=above:{$t^*_\y$}] (t\y0) at (2,\Y-\y*7*\s) {};
     \draw[thick,->,draw=h] (u\y) -- (f\y0) --(t\y0);
     \foreach[count=\j from 0] \i in {1,2,3}{
       \node[smallvertex,label=left:{$f^\i_\y$}] (f\y\i) at (1,\Y-\y*7*\s-2.5*\s-\i*\s) {};
       \draw[thick,->,draw=v] (f\y\j) -- (f\y\i) ;
       \node[smallvertex,label=left:{$t^\i_\y$}] (t\y\i) at (2,\Y-\y*7*\s-\i*\s) {};
       \draw[thick,->,draw=v] (t\y\j) -- (t\y\i) ;
     }
   }
   \foreach \q/\v/\l in {q11/t11/$\bar\ell^1_1$,q12/f21/$\ell^1_2$,q13/f31/$\ell^1_3$
     ,q21/f12/$\ell^2_1$,q22/f22/$\ell^2_2$,q23/f32/$\ell^2_3$
     ,q31/t13/$\bar\ell^3_1$,q32/t23/$\bar\ell^3_2$,q33/t33/$\bar\ell^3_3$}
   {
     \node[smallvertex,label=right:{\l}] (l) at (\q |- \v) {};
     \draw[thick,->,draw=v] (\q) -- (l);
     \draw[thick,->,draw=h] (\v) -- (l);
   }

   \foreach \v/\l in {t12/$\bar\ell^2_1$,t21/$\bar\ell^1_2$,t22/$\bar\ell^2_2$,t31/$\bar\ell^1_3$,t32/$\bar\ell^2_3$,
     f11/$\ell^1_1$,f13/$\ell^3_1$,f23/$\ell^3_2$,f33/$\ell^3_3$}{
     \node[smallvertex,label=right:{\l}] (l) at (q11 |- \v) {};
     \draw[thick,->,draw=h] (\v) -- (l);
   }
   \foreach \a/\b/\c in {z/qx1/h,qx1/qx2/h,qx2/qx3/h,z/u1/v,u1/u2/v,u2/u3/v}
   {
     \draw[thick,->,draw=\c] (\a) -- (\b);
   }

%   \node[smallvertex] (C) at (q11 |- t11) {};
  
 \end{tikzpicture}
 }
\caption{Example of graphs produced by \Cref{const:target} with input formula $\varphi = (\neg x_1 \vee x_2 \vee x_3) \wedge (x_1\vee x_2 \vee x_3) \wedge (\neg x_1 \vee \neg x_2 \vee \neg x_3)$. For greater clarity, the transitive arcs have not been depicted. }
\label{fig:target-construction}
\end{figure}
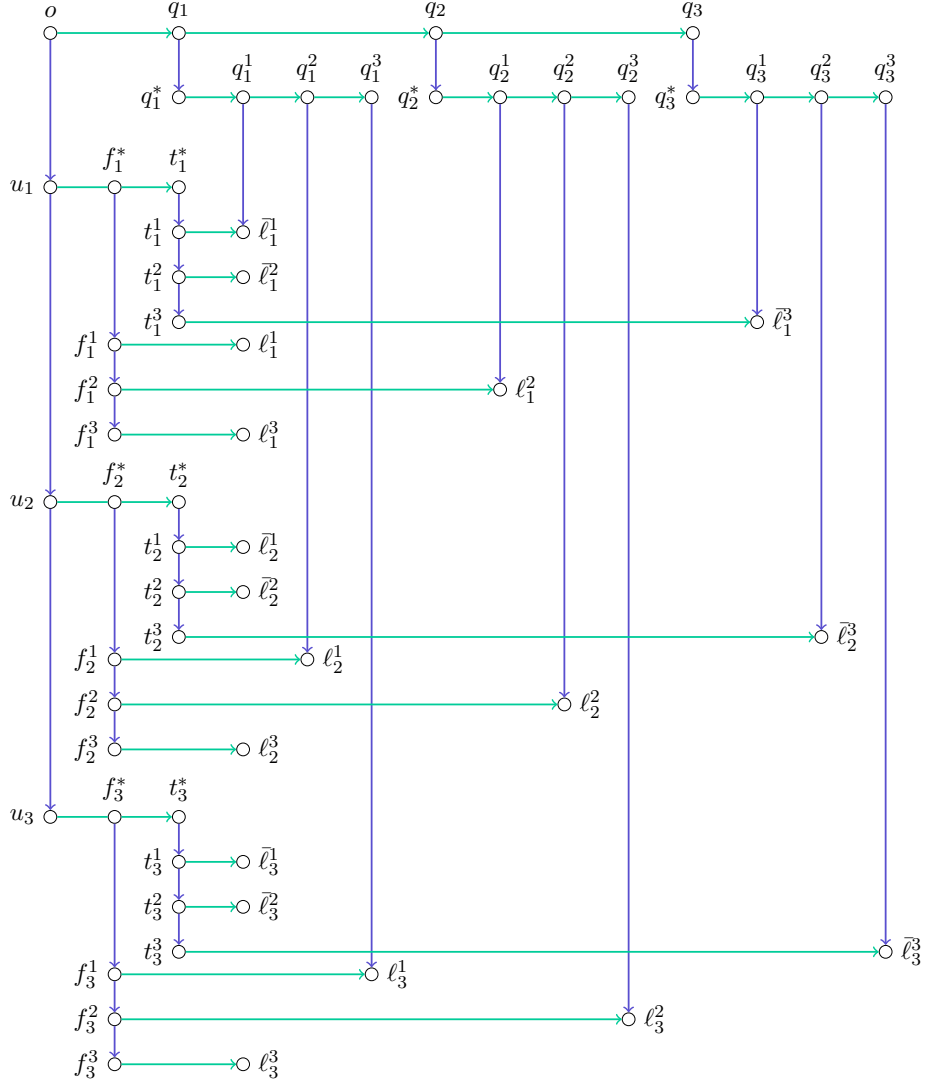

\begin{property}
\label{prop:target_consist}
The graph $G$ defined by Construction \ref{const:target} is a consistent position graph
\end{property}

\begin{proof}
  First, we show that the graph of rows $G^\Pi_R$ does not contain a cycle.  
  Consider the following vertex partition $P_0  \cup \dots \cup P_n$ of $G$:
  \begin{itemize}
    \item $P_0 = \{o,q_i,q^*_i,q^p_i \mid i \in [m], p\in [3]\}$,
    \item $\forall j \in [n], P_j = \{u_j,f^i_j,t^i_j,f^*_j,t^*_j,\ell^i_j,\bar{\ell}^i_j \mid   i \in [m] \}$ 
  \end{itemize}
  Observe that by construction, each row of $G$ is contained in a single part. Moreover, no part $P_k$ contains a cycle of $G^\Pi_R$.
  Now, we show that for any vertical arc $(u,v)$ such that $u\in P_i$, $v \in P_j$ and $P_i \neq P_j$, we have $i <j$. 
  For the arc $(o,u_1)$, we have $o \in P_0$ and $u_1 \in P_1$. For any arc $(u_i,u_{i+1})$, we have $u_i \in P_i$ and $u_{i+1} \in P_{i+1}$. For any arc $(q^p_i,\ell^i_j)$, we have $q^p_i \in P_0$ and $\ell^i_j \in P_j$, with $j\geq 1$. Finally, for any arc $(q^p_i,\bar\ell^i_j)$, we have $q^p_i \in P_0$ and $\bar\ell^i_j \in P_j$, with $j\geq 1$. Hence, it follows that the graph of rows does not contain a cycle.
  
%$$V(G) = \{o,q_i,q^*_i,q^p_i,u_j,f^i_j,t^i_j,f^*_j,t^*_j,\ell^i_j,\bar{\ell}^i_j \mid p\leq 3,  i \leq m ,j \leq n\}.$$
  Further, we show in a similar manner that the graph of columns $G^\Pi_C$ does not contain a cycle.
  Consider the following vertex partition $P_0  \cup \dots \cup P_n$ of $G$:
  \begin{itemize}
    \item $P_0 = \{o,u_j,f^i_j,t^i_j,f^*_j,t^*_j\mid  i \in [m] ,j \in [m]\}$,
    \item $\forall i \in [m], P_i = \{q_i,q^*_i,q^j_i,\ell^i_j,\bar{\ell}^i_j \mid p\in [3] ,j \in [n]\}$.
  \end{itemize}
  Observe that by construction, each column of $G$ is contained in a single part. Moreover, no part $P_k$ contains a cycle of $G^\Pi_C$.
  Now, we show that for any horizontal arc $(u,v)$ such that $u\in P_i$, $v \in P_j$ and $P_i \neq P_j$, we have $i <j$. 
  For the arc $(o,q_1)$, we have $o \in P_0$ and $q_1 \in P_1$. For any arc $(q_i,q_{i+1})$, we have $q_i \in P_i$ and $q_{i+1} \in P_{i+1}$. For any arc $(t^i_j,\bar\ell^i_j)$, we have $t^i_j \in P_0$ and $\bar\ell^i_j \in P_i$, with $i\geq 1$. Finally, for any arc $(f^i_j,\ell^i_j)$, we have $f^i_j \in P_0$ and $\ell^i_j \in P_i$, with $i\geq 1$. Hence, it follows that the graph of columns does not contain a cycle. Hence, By \Cref{lemma:alg inconsistent}, $G$ is consistent.
\end{proof}

Let us now define the graph that corresponds to the pattern. 

\begin{construction}[Pattern, see \Cref{fig:pattern-construction}]
   \label{const:pattern}
Let $\varphi$ be a \TSAT 
formula with $n$ variables and $m$ clauses.
We construct a position graph $H$ with vertex set
$$V(H) = \{w,k_i,k^*_i,s_i,s'_i,v_j,y^*_j,y^i_j,z^i_j \mid i \in [m], j \in [n]\},$$
as follows.
\begin{itemize}
  \item Introduce a transitive horizontal path $A'=(w,k_1,\dots,k_m)$ and a transitive vertical path $B'=(w,v_1,\dots,v_n)$.
  \item For each clause $Q_i$, introduce the horizontal arc $(k^*_i,s_i)$ and the two vertical arcs $(k_i,k^*_i)$ and $(s_i,s'_i)$. 
  \item For each variable $x_j$, introduce the vertical arc $(v_j,y^*_j)$ and a transitive horizontal path $X_j=(y^*_j,y^1_j,\dots,y^m_j)$.
  \end{itemize}
\end{construction}

\begin{figure}[t]
 \scalebox{0.85}{
  \begin{tikzpicture}
   
   \node[smallvertex,label=above:{$w$}] (z) at (0,1) {};
   \foreach \x in {1,2,3}{
     \node[smallvertex,label=90:{$k_\x$}] (qx\x) at (1+2.5*\x,1) {};
     \node[smallvertex,label=left:{$k^*_\x$}] (q\x0) at (1+2.5*\x,0) {};
     \draw[thick,->,draw=v] (qx\x) -- (q\x0) ;
     \foreach[count=\j from 0] \i in {1}{
       \node[smallvertex,label=90:{$s_\x$}] (q\x\i) at (1+2.5*\x+\i,0) {};
       \draw[thick,->,draw=h] (q\x\j) -- (q\x\i) ;
       \node[smallvertex,label=0:{$s'_\x$}] (c\x\i) at (1+2.5*\x+\i,-1) {};
       \draw[thick,->,draw=v] (q\x\i) -- (c\x\i) ;
     }
   }
   \def\s{0.7}
   \def\Y{2}
   \foreach \y in {1,2,3}{
     \node[smallvertex,label=left:{$v_\y$}] (u\y) at (0,\Y-\y*4.5*\s) {};
     \node[smallvertex,label=above:{$y^*_\y$}] (t\y0) at (1,\Y-\y*4.5*\s) {};
     \draw[thick,->,draw=h] (u\y)  --(t\y0);
     \foreach[count=\j from 0] \i in {1,2,3}{
       \node[smallvertex,label=left:{$y^\i_\y$}] (t\y\i) at (1,\Y-\y*4.5*\s-\i*\s) {};
       \node[smallvertex,label=right:{$z^\i_\y$}] (b\y\i) at (q10 |- t\y\i) {};
       \draw[thick,->,draw=v] (t\y\j) -- (t\y\i);
       \draw[thick,->,draw=h] (t\y\i) -- (b\y\i);
     }
   }

   \foreach \a/\b/\c in {z/qx1/h,qx1/qx2/h,qx2/qx3/h,z/u1/v,u1/u2/v,u2/u3/v}
   {
     \draw[thick,->,draw=\c] (\a) -- (\b);
   }
  
 \end{tikzpicture}
 }
\caption{Example of a graph produced by \Cref{const:pattern} with input formula $\varphi = (\neg x_1 \vee x_2 \vee x_3) \wedge (x_1\vee x_2 \vee x_3) \wedge (\neg x_1 \vee \neg x_2 \vee \neg x_3)$. }
\label{fig:pattern-construction}
\end{figure}
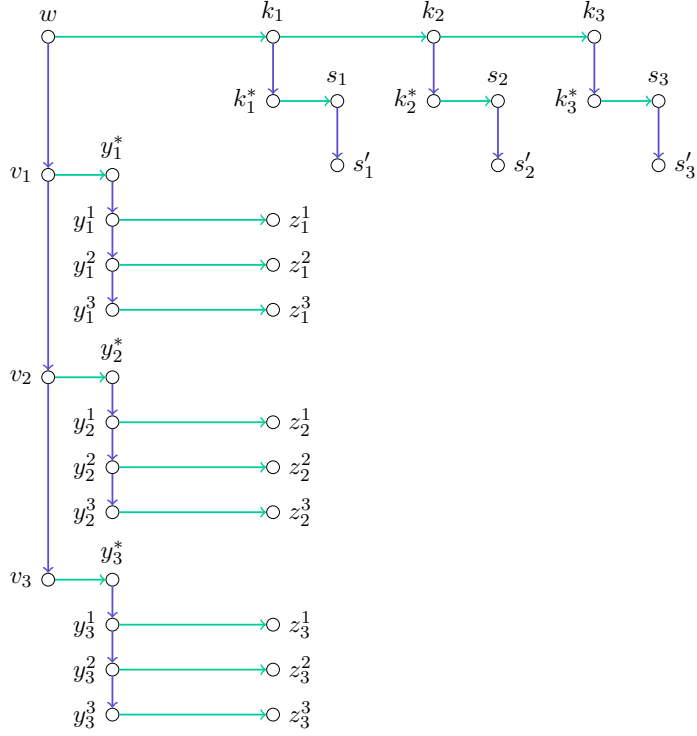

% In the following, we only consider \TSAT formulas in which no literal appears in more than $m-2$ clauses. 
% This restriction does not affect NP-completeness, since one can always add dummy clauses involving new variables to enforce this property.

\begin{property}
The graph $H$ defined by Construction \ref{const:pattern} is a consistent position graph
\end{property}

The proof is similar to that of Property \ref{prop:target_consist} and is even simpler. 

\begin{lemma}
  \label{lemma:nice solution}
  Let $\varphi$ be a \TSAT formula and let $G$ and $H$ be two position graphs produced by \Cref{const:target,const:pattern}, respectively. If $G$ contains a subgraph isomorphic to $H$, then there is an isomorphism $f$ such that:
  \begin{enumerate}
    \item $A'$ is identified with $A$, $B'$ is identified with $B$, and
    \item for each variable $x_j \in \varphi$, $X_j$ is identified with either $T_j$ or $F_j$.
  \end{enumerate}
\end{lemma}
\begin{proof}~\\
\vspace{-0.5cm}  \begin{enumerate}
    \item By construction, $A$ is the only transitive horizontal path containing $m+1$ vertices, so $A'$ is necessarily identified with $A$. It follows that $w$ is identified with $o$ and $B'$ is identified with $B$.
    \item By the previous item, for each variable $x_j \in \varphi$, $v_j$ is identified with $u_j$. Since $u_j$ has two horizontal neighbors $f^*_j$ and $t^*_j$, $y^*_j$ is identified with one of them. It follows that $X_j$ is identified with either $T_j$ or $F_j$.
\end{enumerate}
\end{proof}

\begin{theorem}
  \label{theorem:npc}
  {\sc Induced Subgraph Isomorphism} is NP-complete even when restricted to position graphs.
\end{theorem}
\begin{proof}
The problem belongs to NP, since given an injective mapping $g : V(H)\to V(G)$, one can verify in polynomial time that it defines a label-preserving induced subgraph isomorphism by checking adjacency, non-adjacency, and arc labels. Let $\varphi$ be a \TSAT formula and let $G$ and $H$ be two position graphs produced by \Cref{const:target,const:pattern}, respectively. We show that $\varphi$ is satisfiable if and only if $G$ contains an induced subgraph isomorphic to $H$.
  \begin{itemize}
   \item  Let $\beta$ be a satisfying assignment for $\varphi$. 
We construct the following isomorphism $g$ between $H$ and a subgraph of $G$. 
First, we identify $A'$ with $A$ and $B'$ with $B$.
For each variable $x_j \in \varphi$, if $x_j$ is assigned to \texttt{true}, 
we identify $X_j$ with $T_j$ and for each $i\leq m$, we identify $z^i_j$ with $\bar\ell^i_j$. If $x_j$ is assigned to \texttt{false}, we identify $X_j$ with $F_j$ and for each $i\leq m$, we identify $z^i_j$ with $\ell^i_j$.
For each clause $Q_i$, we identify $k^*_i$ with $q^*_i$.
It remains to assign $g(s_i)$ and $g(s'_i)$ for each clause $Q_i$. 
For every clause $Q_i$, there exists a variable $x_j \in Q_i$ 
such that its assignment in $\beta$ satisfies $Q_i$. Suppose that $x_j$ is the $t^{th}$ variable occurring in $Q_i$.
If $x_j$ occurs positively (resp. negatively) in $Q_i$, 
then by previous mapping, $\ell^i_j$ are not associated yet with any node (resp. $\bar\ell^i_j$ are not associated yet with any node), 
and we can set $g(s_i)=q^t_i$ and $g(s'_i)=\ell^i_j$ (resp. $g(s'_i)=\bar\ell^i_j$). 
    We have constructed an injective mapping $g$ such that for all pairs of vertices $u,v$, $(u,v) \in A(H)$ if and only if $(g(u),g(v)) \in A(G)$ and such that for all arc $(u,v) \in A(H), (u,v)$ and $(g(u),g(v))$ have the same label. It remains to observe that the image of $g$ induces in $G$ exactly the arcs prescribed by $H$. Indeed, by construction, every selected vertex lies on one of the designated clause or variable gadgets, and the only arcs between selected vertices are precisely those used to realize the paths and attachments of $H$. In particular, no additional arc of $G$ connects two vertices in the image of $g$ unless the corresponding vertices are adjacent in $H$. Therefore, $g$ is an induced, label-preserving subgraph isomorphism from $H$ into $G$. Hence, $G$ contains an induced subgraph isomorphic to $H$.

    \item Let $g$ be an isomorphism respecting \Cref{lemma:nice solution}. We construct an assignment for $\varphi$ as follows: for each variable $x_i\in\varphi$, we assign $x_i$ to \texttt{true} if $X_i$ is identified with $T_i$ and we assign $x_i$ to \texttt{false} otherwise. Toward a contradiction, suppose there is a clause $Q_i$ that is not satisfied. Suppose that  $g(s'_i)=\ell^i_j$. By construction, the variable $x_j$ occurs positively in $Q_i$. Since $Q_i$ is not satisfied, $x_j$ is assigned to \texttt{false} and so, $X_j$ is identified with $F_j$. In particular, we have $g(y^i_j)=f^i_j$. But then, since $\ell^i_j$ is the only horizontal neighbor of $f^i_j$, we have necessarily $g(z^i_j)=\ell^i_j$. Thus, $|g^{-1}(\ell^i_j)|\geq 2$ which contradicts $g$ being an isomorphism. If $g(s'_i)=\bar\ell^i_j$, then we also obtain a contradiction by symmetry.
    Therefore, all clauses are satisfied by $\beta$, which concludes the proof.
  \end{itemize}
\end{proof}

\section{Existing Approaches: Relationships and Limitations}
\label{sec:related_works}
This section aims to review several types of works that can be connected to our main representation concern. We argue that the \emph{position space} is a distinct construct that cannot be directly replaced by existing methods.

\subsection{Ordered structures, transitive orientations, and multidimensional patterns}

Our representation is related to several works connecting order theory and graph structure. Partially ordered sets (posets) admit graphical encodings via comparability graphs, where edges reflect comparabilities in the underlying order; equivalently, these are precisely the graphs that admit a transitive orientation~\citep{golumbic1977,golumbic1980}. This perspective has been extensively developed, including recognition and orientation algorithms for transitively orientable graphs~\citep{golumbic1977,golumbic1980}. Position graphs may thus be viewed as a constrained, edge-labeled analogue of such structures: their horizontal and vertical labels restrict which transitive orientations correspond to valid spatial interpretations and impose compatibility conditions beyond transitivity alone.

A complementary point of view is developed in the theory of poset dimension, defined as the minimum number of linear extensions whose intersection yields the poset. Dimension theory connects order structure to geometric realizability in $\mathbb{R}^d$ under the product order and provides criteria for when an ordering admits a low-dimensional coordinate representation~\citep{trotter1992}. 

Although related in spirit to two-dimensional order embeddings, our framework differs by explicitly separating two alignment relations and enforcing cross-relation consistency constraints not captured by standard dimension arguments. More precisely, a poset has dimension~2 when its order relation can be expressed as the intersection of two linear extensions~\citep{trotter1992}. In such settings, the two linear orders 
represent different views of the same underlying hierarchy, and a pair 
$(x,y)$ is comparable only if both extensions agree on their ordering. In contrast, 
a position space defines two independent strict partial orders, 
$<_h$ and $<_v$, each with a distinct spatial interpretation (horizontal versus 
vertical). Moreover, the exclusivity condition of Definition~\ref{def:compatibility} prevents a pair of tokens from being comparable in both relations. As a consequence, position spaces 
do not correspond to dimension-2 posets: instead of intersecting two total orders 
to define a single order relation, they maintain two disjoint precedence structures 
that cannot be merged into a unified poset.

Classical grid or geometric embeddings of posets (e.g., using dominance orders, 
rectangle representations, or boxicity embeddings) aim to represent a single 
partial order in two-dimensional space by assigning each element coordinates or 
regions whose geometric relations preserve the order. Although consistent position 
spaces also admit a two-dimensional layout through their row/column assignment, 
this embedding does not represent a global poset order. Instead, it preserves 
$<_h$ and $<_v$ independently, and the exclusivity condition prevents the existence 
of a single coordinate-wise order capturing both. Thus, despite similarities, position spaces are structurally distinct from classical grid embeddings: they encode two orthogonal alignment relations rather than a two-dimensional realization of a single poset.

Let us also remark that the direct product of two linear orders defines the classical $2$-dimensional grid poset: $(a,b) \leq (c,d)$ if and only if $a \leq c$ and $b \leq d$. This coordinate-wise comparison requires that ordered pairs satisfy both dimensional constraints. 
Position spaces behave differently: comparability is defined separately by one alignment relation, either horizontal or vertical, but never both. Hence, position spaces do not directly form a 
subclass of product orders.

Finally, pattern detection in ordered structures also relates to  permutation patterns and the avoidance/containment paradigm, where containment is defined through preservation of relative order in subsequences~\citep{bonaPermutations2004,marcusTardos2004}. Classical permutation patterns operate over a single total order; in contrast, our setting involves two interacting order relations together with labeled edges, connecting it more closely to multidimensional generalizations of containment than to the one-dimensional theory of permutation patterns.

\subsection{Subgraph and Induced Subgraph Isomorphism}

Many works investigate graph pattern matching—particularly subgraph and induced-subgraph isomorphism, as the task of identifying occurrences of a smaller pattern inside a larger structure, a problem that becomes especially intricate for labeled graphs and directed acyclic targets. Classical and modern practical approaches emphasize \emph{filter-and-search} pipelines and constraint-programming–inspired pruning strategies, especially for large graphs with small pattern queries~\citep{mccreesh2018}. Recent developments introduce new heuristics and structural refinements tailored for labeled or directed settings, as illustrated by ~\citep{AsilerYG22}, which provides an efficient subgraph‑isomorphism framework for large graph databases, or 
FiPE~\citep{LuZZZ25}, a subgraph‑matching algorithm that eliminates redundant computations by defining fine‑grained equivalence relations over vertex pairs and multi‑vertex patterns. The Glasgow Subgraph Solver further exemplifies state-of-the-art constraint-guided backtracking with support for both induced and non-induced variants and parallel execution~\citep{McCreeshP020}. LAD2025~\citep{Solnon26} also provides a modern constraint‑based solver for the subgraph‑isomorphism problem, strengthening domain filtering and constraint propagation to achieve competitive performance on a broad range of benchmark instances. 

Since position graphs in our setting are directed and acyclic, the problem naturally falls within labeled pattern matching on DAG-like structures, where the induced requirement—preserving both adjacency and non-adjacency—typically increases difficulty relative to non-induced matching; empirical studies show that induced and non-induced instances can behave quite differently in practice, even when both formulations remain NP-complete~\citep{mccreesh2018}. Induced subgraph isomorphism itself is a classical NP-complete problem~\citep{garey1979}. Despite this worst-case hardness, decades of practical work—including recent large-scale experiments—demonstrate that many real-world instances are solvable using modern pruning and search strategies, although hard families persist and motivate the exploitation of domain-specific structure, particularly in specialized graphical representations such as DAG-like position graphs.

\subsection{Qualitative Spatial Reasoning}
Qualitative Spatial Reasoning (QSR) \citep{ChenCLWOY15}  allows reasoning about spatial relationships without relying on precise numerical coordinates, unlike quantitative approaches. QSR focuses on symbolic and relational descriptions of space, which are particularly useful in domains where exact measurements are unavailable, unnecessary, or computationally expensive.

One of the most known frameworks in QSR is RCC8 (Region Connection Calculus) \citep{Randell92}, which defines eight mutually exclusive and jointly exhaustive base relations: Disconnected (DC), Externally Connected (EC), Partially Overlapping (PO), Equal (EQ), Tangential Proper Part (TPP), Non-Tangential Proper Part (NTPP), and their inverses (TPPi, NTPPi). These relations capture fundamental aspects of spatial configuration without relying on precise numerical coordinates, making RCC8 particularly suitable for applications \citep{CohnBGG97} in geographic information systems, spatial databases, and AI reasoning tasks . These rules  focus on topological relationships between regions (e.g., disconnected, overlapping, and contained). RCC8 abstracts metric details but still assumes continuous spatial regions. The relations provide a logical calculus system for reasoning about connectivity and containment.

Our approach focuses on the relative positioning of discrete tokens using horizontal and vertical alignment relations. The spatial organization is represented by strict partial orders rather than topological predicates. We also emphasize rows and columns and their compatibility, enabling declarative modeling of layout-like structures. Hence, QSR is more region-centric, whereas the proposed approach is token-centric.

However, RCC8 does not encode direction or order (e.g., “A is left of B” or “C is above D”). To recover the horizontal (h) and vertical (v) edge semantics of a Position Graph, we should combine RCC8 with a directional calculus such as the Cardinal Direction Calculus (CDC) \citep{Frank92} or Rectangle Algebra (RA) \citep{MukerjeeJ90}.Example \ref{ex:RCC8} illustrates these relationships between position graphs and description logics. 

\begin{example}[Mapping Position Graphs to Qualitative Constraints]
\label{ex:RCC8}
   
Assume a Position Graph $({\cal T},{\cal A},{\cal E})$. We may construct a qualitative spatial reasoning constraint network with one region variable per token. For each token $t\in {\cal T}$, create a region variable $R_t$. We can use RCC8 to forbid overlaps where the graph prohibits them, and to allow/force contacts where alignment suggests adjacency. If, in the intended layout, two different columns in the same row must not overlap, set:
\[
\forall\,t,t' \text{ in same row but different columns:}\quad (R_t,R_{t'}) \in \{DC,EC\}.
\]
Choose DC if gaps are mandatory; choose EC if touching is permitted (adjacent cells). Analogous reasoning is performed for rows:
\[
\forall\,t,t' \text{ in same column but different rows:}\quad (R_t,R_{t'}) \in \{DC,EC\}.
\]
If the model needs cells or rows/columns as container regions, we may use TPP/NTPP to encode that a token lies properly inside its row/column region:
\[
(R_t, R_{row(t)}) \in \{TPP,NTPP\},\qquad
(R_t, R_{col(t)}) \in \{TPP,NTPP\}.
\]
TPP vs NTPP distinguishes whether the token touches the boundary of the container.

To capture order (left--right for $h$, above--below for $v$), add CDC (or RA) constraints. For every horizontal edge $E(t,t')=h$:
\[
\mathrm{CDC}(R_t,R_{t'}) \in \{\mathrm{W},\mathrm{W\!-\!NW},\mathrm{W\!-\!SW}\},
\]
i.e., $t$ lies west of $t'$. (You can pick a single cardinal relation W if you want strict west-of.). For every vertical edge $E(t,t')=v$:
\[
\mathrm{CDC}(R_t,R_{t'}) \in \{\mathrm{N},\mathrm{N\!-\!NE},\mathrm{N\!-\!NW}\},
\]
i.e., $t$ lies north of $t'$. (Likewise, choose N for strict north-of)

Mixing RCC8 with CDC is, of course, computationally expensive, as shown in \citep{LiuLR09}.
\end{example}

Note that \citep{Rohrig94} proposed a  theory for qualitative spatial reasoning by expressing spatial concepts through simple order relations on low-dimensional structures. 

\section{Conclusion}

In this work, we introduced position spaces and their graph representation, position graphs, as a lightweight qualitative framework for reasoning about the relative organization of discrete tokens using two strict partial orders capturing horizontal and vertical precedence.
By imposing a chain (no-branching) condition on successors and predecessors and a compatibility requirement that forbids mixed horizontal/vertical chains between tokens already ordered in one direction, we formalized when such positional information admits a coherent two-dimensional interpretation in terms of rows and columns.
We then provided a structural characterization of consistency for position graphs and showed how consistency can be checked constructively.
Building on this characterization, we designed an R/C embedding algorithm that assigns row and column indices to tokens whenever the position space is consistent, and proved that it runs in linear time in the size of the graph.

Beyond consistency, we studied structural pattern discovery as a graph-matching problem and modeled it as label-preserving induced subgraph isomorphism between position graphs. We established a strong hardness boundary by proving that induced subgraph isomorphism remains NP-complete even when restricted to consistent position graphs, despite the substantial structural constraints imposed by the position space axioms.
This separation clarifies which reasoning tasks can be solved reliably at scale and which will generally require heuristics or specialized solvers in practical applications such as document layout analysis.

As future work, an important direction is to identify maximal tractable fragments of the induced matching problem within position graphs, enabling efficient discovery of useful pattern families while retaining the formal benefits of the proposed representation.

\bibliographystyle{empty}
%\bibliography{./biblio}


\end{document}